\documentclass{article}

\usepackage{PRIMEarxiv}

\usepackage[utf8]{inputenc} 
\usepackage[T1]{fontenc}    
\usepackage{hyperref}       
\usepackage{url}            
\usepackage{booktabs}       
\usepackage{amsfonts}       
\usepackage{nicefrac}       
\usepackage{microtype}      
\usepackage{lipsum}
\usepackage{fancyhdr}       
\usepackage{graphicx}       
\graphicspath{{media/}}     
\usepackage{graphicx}
\usepackage{hyperref}
\usepackage{multirow}
\usepackage{rotating} 
\usepackage{array}
\usepackage{svg}
\usepackage{amsfonts}
\usepackage{amsmath}
\usepackage{amssymb}
\usepackage{amsthm}
\usepackage{color}
\usepackage{tikz}
\usepackage{subfig}
\newtheorem{theorem}{Theorem}
\newtheorem{definition}{Definition}
\newtheorem{lemma}{Lemma}
\newtheorem{corollary}{Corollary}
\newtheorem{remark}{Remark}
\newtheorem{proposition}{Proposition}

\definecolor{mydarkblue}{RGB}{0,0,139} 

\title{Comparative Analysis of Indicators for Multiobjective  Diversity Optimization
	\thanks{The authors gratefully acknowledge financial support by NWO,  The Netherlands, TOP Grant 18036 : Excellent Buildings: Realistic geometries for optimal structural, thermal, and lighting performance.}}

\author{
	Ksenia Pereverdieva \\
	Leiden Institute of Advanced Computer Science \\
	Leiden University  \\
	Leiden, The Netherlands \\
	\And
	André Deutz \\
	Leiden Institute of Advanced Computer Science \\
	Leiden University \\
	Leiden, The Netherlands \\
	\And
	Tessa Ezendam \\
	Dept. of Built Environment \\
	Eindhoven University of Technology \\
	Eindhoven, The Netherlands \\
	\And
	Thomas Bäck \\
	Leiden Institute of Advanced Computer Science \\
	Leiden University \\
	Leiden, The Netherlands \\
	\And
	Hèrm Hofmeyer \\
	Dept. of Built Environment \\
	Eindhoven University of Technology \\
	Eindhoven, The Netherlands \\
	\And
	Michael T.M. Emmerich \\
	Faculty of Information Technology, University of Jyväskylä \\
	Jyväskylä, Finland \\
	michael.t.m.emmerich@jyu.fi
}

\begin{document}
	\maketitle
	
	\begin{abstract}
		Indicator-based (multiobjective) diversity optimization aims at finding a set of near (Pareto-)optimal solutions that maximizes a diversity indicator, where diversity is typically interpreted as the number of essentially different solutions. Whereas, in the first diversity-oriented evolutionary multiobjective optimization algorithm, the NOAH algorithm by Ulrich and Thiele, the Solow Polasky Diversity (also related to Magnitude) served as a metric, other diversity indicators might be considered, such as the parameter-free Max-Min Diversity,  and the Riesz \(s\)-Energy, which features uniformly distributed solution sets. In this paper, focusing on multiobjective diversity optimization, we discuss different diversity indicators from the perspective of indicator-based evolutionary algorithms (IBEA) with multiple objectives. We examine theoretical, computational, and practical properties of these indicators, such as monotonicity in species, twinning, monotonicity in distance, strict monotonicity in distance, uniformity of maximizing point sets, computational effort for a set of size~\(n\), single-point contributions, subset selection, and submodularity. We present new theorems---including a proof of the NP-hardness of the Riesz \(s\)-Energy Subset Selection Problem---and consolidate existing results from the literature.
		In the second part, we apply these indicators in the NOAH algorithm and analyze search dynamics through an example. We examine how optimizing with one indicator affects the performance of others and propose NOAH adaptations specific to the Max-Min indicator.
	\end{abstract}
	
	
	\maketitle

	\section{Introduction}
	
	In the field of multi-objective optimization and constraint satisfaction, the difficulty often goes beyond simply finding optimal or feasible solutions. An increasingly important factor is the need to produce a diverse set of near-optimal or feasible solutions. This process, referred to as diversity optimization, seeks to offer decision makers a range of potential solutions that are both close to optimality and varied from one another. The reasoning behind this is that certain criteria, which determine the 'best' solution, are often subjectively assessed and may not easily be expressed mathematically. An example motivating our research is building design, which combines subjective criteria like aesthetics with objective criteria such as energy performance and structural compliance\cite{pereverdieva2023prism,ezendam2022partitioning}.
	
	An early attempt to meet these challenges are cluster-based genetic algorithms by Parmee \cite{parmee2000evolutionary}, which aim to detect different high-performance regions in design optimization. Another important contribution is Evolutionary Algorithms for Generating Alternatives (EAGA) by Emily Zechman \cite{zechman2005comparative}, which uses approaches to generate diverse sets of alternative solutions by relaxing the optimality requirement and also considering solutions that are near optimal, while near-optimality is expressed by an acceptability threshold. The Omni-Optimizer, introduced by Deb and Tiwari \cite{deb2005omni}, focuses on multiglobal, multiobjective optimization and the detection of different optima, but it follows the design paradigm of the NSGA-II algorithm, prioritizing convergence over diversity. 
	Quality Diversity Optimization (QD), as explored, for example, by Hagg \cite{hagg2020quality}, offers an alternative direction on this by providing best-of-class solutions in a design space that is partitioned into subspaces by means of certain features. Compared to the previously mentioned approaches, it requires that the user defines not only an optimization model, but also a feature space. 
	In the field of multiobjective optimization, decision space diversity has gained attention, beginning with the seminal work of Rudolph et al. \cite{rudolph2007capabilities}. This was followed by various algorithmic proposals aimed at enhancing decision space diversity through niching techniques (\cite{shir2009enhancing}), as well as an early attempt to incorporate the indicator-based approach using a cuboid covering metric \cite{ulrich2010integrating}. Additionally, application-oriented studies, such as the work by Preuss et al. \cite{preuss2010decision}, have contributed to this growing body of research. 
	Using the indicator-based evolutionary algorithm paradigm, Ulrich and Thiele \cite{ulrich2011diversity} addressed the "diversity optimization" challenge by developing techniques that seek to directly maximize the Solow-Polasky (SP) diversity indicator of a population in an algorithm they term NOAH. We seek to further explore the domain of indicator-based diversity optimization, in particular by using alternative diversity indicators.
	
	
	
	In this paper we focus on optimizing indicators based on the distance between solutions in a set, namely Max-min Diversity, Riesz s-Energy, and Solow Polasky Diversity. At the same time, 
	we exclude indicators that  
	do not tend to distribute uniformly, such as total distance (\cite{ulrich2011diversity}, Fig. \ref{fig:counterintuitive}), are computationally intractable, such as the Weitzman diversity\cite{weitzman1992diversity}, and that focus more on the aspect of species abundance, such as entropy\cite{leinster2013magnitudeOfMetricSpaces}. See \cite{basto2017survey} for a broader survey.
	We examine the theoretical properties of diversity indicators, focusing on aspects like \textit{monotonicity in species}, ensuring that adding a species does not reduce diversity; \textit{twinning}, examining the effect of identical or very similar species; \textit{monotonicity in distance}, ensuring that increasing the distance between species increases diversity; \textit{strict monotonicity in distance}; \textit{uniformity of maximizing point sets}; \textit{computational effort} for a set of size~\(n\); \textit{single-point contributions}; \textit{subset selection}; and \textit{submodularity}. As an important side result, we prove that the Riesz \(s\)-Energy subset selection is NP-hard in general metric spaces and provide new results in computing single-point contributions.

	In the empirical analysis we investigate how various diversity indicators perform within the NOAH algorithm for multiobjective optimization problems. Originally designed for optimization with SP diversity, we tested it with other indicators, in particular Max-Min Diversity and Riesz s-Energy. 
	In addition, we study how the optimization of one of the diversity indicators affects the performance of the other diversity indicators. 
	
	
	The paper is structured as follows. Section \ref{sec:background} analyzes the properties of the selected diversity indicators. Section \ref{sec:noah} empirically examines these indicators in the NOAH algorithm and discusses specific adjustments. We conclude by summarizing key findings and future research directions in Section \ref{sec:discussion_and_outlook}.

	
	
	\section{Background and Theoretical Properties}
	\label{sec:background}
	
	This section derives theoretical properties for selected diversity indicators. While at first glance it may seem inconsequential which diversity indicator is chosen, as they all, in some way, reward diversity, it is important to take a deeper look and conduct a theoretical analysis. See Fig. \ref{fig:counterintuitive} for an example where intuition may be misleading. The example shows that comparing two sets based on the minimal gap -- or the Max-Min Diversity -- grossly neglects the benefits of distancing points that do not form a closest pair.  The example also shows that maximizing the sum of pairwise distances does not succeed in distributing points evenly; quite the opposite, it leads to clustered data sets (cf. \cite{ulrich2011diversity}).
	
	In the sequel, we will use the notion of similarity space, and the diversity indicators we will study will also be defined for finite similarity spaces; the definition is as follows.
	\begin{definition} \label{def:similarity-space}
		The pair
		$(S,d)$ is  a \emph{similarity space} if $S$ is a set and $d: S\times S \rightarrow \mathbb{R}$ is a function satisfying
		\begin{enumerate}
			\item $\forall x \in S, \forall y \in S: d(x,y) \geq 0$, 
			\item $\forall x  \in S: d(x,x) =0 $, and
			\item $\forall x \in S, \forall y \in S: d(x,y) = d(x,y)$.
		\end{enumerate}
		As usual we define the distance of $s \in S$ to a finite subset $X$ of $S$ as follows (notation $d(s,X)$): $$d(s, X) := \min \{ d(s,x)\, |\, x \in X \}.$$ 
		The distance matrix (notation: $\mbox{dm}$ or $\mbox{dm}(X)$ or $\mbox{dm}(X, d)$) of  $X = \{x_1,\cdots, x_n \} \subseteq S$ is $\mbox{dm}_{ij} := d(x_i,x_j)$.  
	\end{definition}
	\begin{remark}
		Note that some authors require $\forall x \in S, \forall y \in S \mbox{ such that } x \neq y: d(x,y) > 0 $ for $(S, d)$ to be a similarity space; in other words, $d$ does not have to satisfy the triangle inequality but satisfies the remaining axioms of a metric space. On reason has to do whether a similarity matrix is positive definite or not. Another reason for this requirement is that $d(x, X) = 0$ implies $x \in X$ for a subset of $S$. For a pair $(S,d)$ satisfying Definition \ref{def:similarity-space} we have $x \in X$ implies $d(x,X) = 0 $ but in general the converse is not true. See example below. 
		
		Example: Let $S=\{x_1,x_2,x_3\}$ be a similarity space with\\
		$d(x_1,x_2)=d(x_2,x_1)=1, d(x_1,x_3) = d(x_3,x_1)=1, d(x_2,x_3)=d(x_3,x_2)=0$, and for each $x$ in $S$: $d(x,x)=0$. Let $X =  \{x_1, x_3\}$ then $x_2 \not \in X$ despite the fact that $d(x_2, X) = 0$.  
		
		In the sequel we assume that our similarity spaces 
		are non-trivial. A similarity space is trivial if for any pair of points the distance is zero. This entails that the metric spaces we consider have cardinality bigger than one. 
	\end{remark}
	Diversity can be understood through two key aspects: similarity and abundance\cite{tomLeinster2021entropyDiversityAxiomatic}. Similarity-based measures assess how many distinct points exist in a set and how far apart they are. Abundance-based measures focus on how evenly points are distributed across predefined groups. Some diversity measures emphasize one aspect more than the other. Here, we will focus on similarity-based diversity measures, which are often most useful for diversity optimization.
	\begin{itemize}
		\item \textbf{Max-Min Diversity}: $D_{\mbox{max-min}}(X) \, =  \, \min\{d(x,y) \, | \, x \in X , y \in  X , x \neq y \}.$ 
		Algorithms using this indicator aim to maximize the value of this indicator, leading asymptotically to uniformly spaced sets, and is often used for diversity assessment in metaheuristics like tabu search and GRASP \cite{porumbel2011simple}.   
		\item \textbf{Riesz s-Energy}: 
		$E_s(X) = \sum_{x_i, x_j \in X, d(x_i,x_j)\neq 0} \frac{1}{d^s (x_i x_j)}$.\\
		(In case for all pairs of elements in $X$ the distance is zero, $E_s(X) := 0$.)
		Defines diversity as the total energy of a system of particles with varying species composition, and yields uniformly distributed sets \cite{hardin2005minimal}. Pairwise potential functions were proposed for evenly covering Pareto fronts in multiobjective optimization 
		\cite{purshouse2013generalizedDecomposition}, and specifically employing Riesz s-Energy  in \cite{falcon2019cri,Falcon-Guillermo2024RieszEnergy}.
		\item \textbf{Solow Polasky Diversity}: $SP(X,\theta)= \sum_{i, j \in \{1,\dots, n \}}(\mbox{sim}(X, \theta)^{-1})_{ij}$, $\theta > 0$\\ $X = \{x_1, \dots, x_n\}$ where
		$\mbox{sim}(X,\theta)$ is the similarity matrix of $X$ defined by $\mbox{sim}(X, \theta)_{ij} := \exp(-\theta d(x_i,x_j))$.  
		Solow and Polasky \cite{solow1993measurement} base diversity on species correlation where a species is viewed as a collection of traits.  
		It is based on the fundamental mathematical concept of magnitude of a finite metric  space in  \cite{huntsman2023diversity,tomLeinster2021entropyDiversityAxiomatic} which was introduced in \cite{solowPolasky1993MeasuringBioDiv} and in \cite{leinster2013magnitudeOfMetricSpaces}. 
		Leinster \cite{tomLeinster2021entropyDiversityAxiomatic} (p. 212) describes magnitude as "the effective number of points" in a finite metric space, or more precisely, "the effective number of completely separate points.".
		\begin{enumerate}
			\item Let $\mbox{sim}(X, \theta)$ be the  similarity matrix associated to $X$ and $\theta$ (as defined above).  Note that in general this matrix does not have to be invertible. Note that the similarity matrix (for finite metric (sub)spaces as well as for similarity (sub)spaces) is symmetric.
			\item Next we describe when a similarity matrix has a magnitude. 
			\emph{Assume} that there exists a column vector 
			$w =
			\left [
			\begin{array}{cccc}
				w_1 &
				w_2 &
				\cdots &
				w_n        
			\end{array}
			\right ]^{\top} 
			\in \mathbb{R}^n
			$ 
			($n \times n$ is the size of the similarity matrix) such that
			$\mbox{sim}(X, \theta)\cdot w = \mathbf{1}$ (where 
			$ \mathbf{1} = 
			\left [
			\begin{array}{cccc}
				1 &
				1 &
				\cdots &
				1        
			\end{array}
			\right ]^{\top}
			$ )
			(Such a $w$ is called a weighting. For symmetric matrices the existence of a weighting implies the existence of a coweighting $v$ and conversely, where a coweighting is a row vector $v$ such that $v \cdot \mbox{sim}(X, \theta) = 
			\left [\begin{array}{cccc}
				1 & 1 & \dots & 1
			\end{array}
			\right ]
			$
			).  If matrix has a weighting and coweighting than it is easy to show that the sum of the components of each weighting is an invariant, that is, if $w$ and $w'$ are two weightings, then 
			$\sum_{i = 1}^n w_i = \sum_{i = 1}^n w'_i $. 
			This sum is called the magnitude of the matrix and the magnitude of the similarity matrix of a finite metric space (or of the finite similarity space) is the magnitude of its similarity matrix. 
			Note that in case the similarity matrix is invertible then there is a unique weighting, namely
			$(\mbox{sim}(X, \theta))^{-1}\cdot \mathbf{1}$ and the magnitude is equal to $\sum_{i, j \in \{1,\dots, n \}}(\mbox{sim}(X, \theta)^{-1})_{ij}$.
			\item The Solow-Polasky diversity indicator $\mbox{SP}(X,\theta)$, is the magnitude of the similarity matrix of $X$ in case the magnitude is defined otherwise it is undefined. For instance, in case the similarity matrix is positive definite the magnitude is defined.   
		\end{enumerate}    
	\end{itemize}

	\begin{figure}[t]
		\centering
		\begin{tikzpicture}
			\draw[thick,blue] (0,0) -- (2,0);
			\foreach \x/\name in {0/a, 0.2/b, 1.8/c, 2/d} {
				\fill (\x,0) circle (2pt);
				\node[above] at (\x,0) {\textbf{\name}};
			}
			\node[left] at (0,0) {(A)\, };
			
			\draw[thick,blue] (3,0) -- (5,0);
			\foreach \x/\name in {3/a, 3.2/b, 4/c, 4.2/d} {
				\fill (\x,0) circle (2pt);
				\node[above] at (\x,0) {\textbf{\name}};
			}
			\node[left] at (3,0) {(B)\, };
		\end{tikzpicture}
		\quad
		\begin{tikzpicture}
			\draw[thick,blue] (0,0) -- (2,0);
			\foreach \x/\name in {0/a, 0.66/b, 1.33/c, 2/d} {
				\fill (\x,0) circle (2pt);
				\node[above] at (\x,0) {\textbf{\name}};
			}
			\node[left] at (0,0) {(C)\, };
		\end{tikzpicture}
		\caption{\label{fig:counterintuitive} 
			Max-min is identical for both (A) and (B), although all pairwise distances increase from (A) to (B), except for nearest neighbors, which remain unchanged. 
			The Sum indicator (aka Total Distance Indicator, sum of all pairwise distances in the set) for (A) is larger than for (C) though intuitively the dissimilarity is much larger for (C). Also the spread and evenness are much better in (C).
		}
	\end{figure}
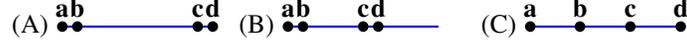

	For the three diversity indicators we introduced below the  Definitions
	\ref{def:contribution}, 
	\ref{def:monotonicity-in-species},
	\ref{def:twinning},  
	\ref{def:monotonicity-in-distance}, 
	\ref{def:D-is-submodular},
	\ref{def:isometry-invariant}, and
	\ref{def:subset-selection-complexity}
	of the properties of these indicators we want to study are given. 
	In the definitions below $D$ denotes a diversity indicator and ($X = \{ x_1, x_2,\cdots, x_n \}$ a subset of $S$, the similarity space). 
	As preliminary steps we recall the definition of a set function to be submodular via the following lemma (Lemma \ref{lemma:submodular}) -- any set function satisfying the properties of this lemma is submodular. Secondly we define for $x \in S, X \subset S$ to be a  duplicate with respect to a subset $X$:
	\begin{definition}
		An element $x \in S$ is a  duplicate with respect to $X \subset S$ iff $\exists y \in X$ such that $d(x,y) = 0$ and $\forall z \in X: d(x,z) = d(y,z)$. 
	\end{definition}
	\begin{remark}
		In metric spaces, it follows that a duplicate $x$ with respect to $X$ is a member of $X$  (unless bags are allowed) and
		any $y \in Y \subseteq S$ is a duplicate with respect to $Y$.
	\end{remark}
	\begin{lemma} \label{lemma:submodular}
		Let S be a \emph{finite} set and $f: \mathcal{P}(S) \rightarrow \mathbb{R}$. 
		The following three properties of $f$ are equivalent.
		\begin{enumerate}
			\item $\forall Y \in \mathcal{P}(S), \forall Z \in \mathcal{P}(S) \mbox{ such that } Y \subseteq Z, \forall x \in S\setminus Z $ the following holds.
			$( f(Y\cup \{x \}) - f(\{Y\}) ) \geq (f(Z\cup \{ x\}) - f(Z) )$,
			\item $\forall T_1 \in \mathcal{P}(S), \forall T_2 \in \mathcal{P}(S):  f(T_1)+f(T_2) \geq f(T_1 \cup T_2) + f(T_1 \cap T_2) $, 
			\item $\forall Y \in \mathcal{P}(S), \forall x_1 \in S \setminus Y, \forall x_2 \in S \setminus Y \mbox{ such that } x_1 \neq x_2$ the following holds:
			$f(Y \cup \{ x_1\}) + f(Y \cup \{ x_2\}) \geq f(Y \cup \{ x_1, x_2\}) + f(Y)$
		\end{enumerate}
	\end{lemma}
	
	\begin{definition}[Theoretical Properties of Diversity Indicators]
		Next we define the theoretical properties of diversity indicators that will be investigated for selected indicators in this work.
		\begin{enumerate}    
			\item 
			Let $x \in S$. Then the individual contribution of $x$ to the diversity $D$ with respect to $X$, notation $\mbox{indD}(x)$, is defined as follows: $\mbox{indD}(x) := D(X \cup \{x \}) - D(X\setminus \{ x \})$.
			
			The individual contribution of all elements of $X$ to the diversity $D$, notation $\mbox{allC}(X)$, is the following n-tuple where $n$ is the size of $X$: $\mbox{allC}(X) := (\mbox{indD}(x_1),\mbox{indD}(x_2), \cdots, \mbox{indD}(x_n) ) $. \label{def:contribution}
			\item 
			Monotonicity in Varieties (aka monotonicity of set functions). A diversity indicator $D$ is Monotonic in Varieties, if $D(S \cup \{ b \} ) > D(S)$, when $b \not \in S$ (in metric space this is equivalent to $\min(\{ d(b,s) \, | \, s \in S \}) > 0$). (In similarity spaces: $d(x, X)> 0$ implies $x \not \in X$ but in the general the converse does not hold.) \label{def:monotonicity-in-species}
			\item 
			Twinning: In case $x$ is a duplicate with respect to $X \subset S$, then $D(X \cup \{ x\}) = D(X)$.  
			\label{def:twinning}
			\item 
			Monotonicity in Distance: for all $A' \subseteq M$, for all $A \subseteq M$ such that $A'$ and $A$ are finite and $|A'| = |A|$, and there is a bijection of $A'$ onto $A$ such that $d(a_{i'}, a_{j'}) \geq d(a_i,a_j)$, then $D(A') \geq D(A)$. Moreover, in case at least one distance $d(a_{i'}, a_{j'})$ is strictly larger than $d(a_i,a_j)$, the diversity of $A'$ should be strictly larger than the diversity of $A$, that is, $D(A') > D(A)$. The latter is called strict monotonicity in distance.\label{def:monotonicity-in-distance}
			\item An indicator  $D$ satisfying one (and thus all) of the properties stated in Lemma \ref{lemma:submodular} is \emph{submodular}. For instance, 
			Let $D: \mathcal{P}(S) \rightarrow \mathbb{R}$ be a diversity indicator, then $D$ is \emph{submodular}, iff $\forall T_1 \in \mathcal{P}(S), \forall T_2 \in \mathcal{P}(S):  D(T_1)+D(T_2) \geq D(T_1 \cup T_2) + D(T_1 \cap T_2) $.\label{def:D-is-submodular} 
			\item D is isometry invariant iff for each isometry $\iota: S \rightarrow S$ and for each subset $X$ of $S$: $D(X) = D(\iota (X))$, where $\iota: S \rightarrow S$ is an isometry of $S$ iff $\forall s_1 \in S, \forall s_2 \in S: d(s_1,s_2) = d(\iota(s_1), \iota(s_2)).$ (Note that for infinite spaces $\iota$ is in general injective but not surjective.) \label{def:isometry-invariant}
			\item The complexity of the $k$-subset selection problem: given a $X \subset S$ find a subset $\tilde{X}$ of $X$ of size $k$  such that $D(\tilde{X})$ is a global optimum of the set $\{ D(X') \, | \, X' \in \mathcal{P}(X) \mbox{ and } |X'| = k \}$ \label{def:subset-selection-complexity}
		\end{enumerate}
	\end{definition}
	\subsection{Theoretical properties of the considered diversity indicators }
	In the following we analyse the properties of the indicators the conclusions will be summarized in Table \ref{table:propertiesOfDivIndicators}.
	\subsubsection{Theoretical Properties of the Riesz s-energy indicator}
	\begin{enumerate}
		\item \emph{Twinning} 
		For similarity spaces twinning does not hold in general for take $b \in S$ such that $b$ is a duplicate wrt $X \subset S$. Then in general it is true that $E_s(X \cup \{ b \}) \geq E_s(X)$. But it can happen that $E_s(X \cup \{ b \}) > E_s(X)$. 
		For metric spaces twinning holds vacuously (unless bags are allowed). For metric as well as similarity spaces adding an appropriate 'almost' twin will for $C > 0$  may lead to   $E_s > C$, that is, it can happen that adding an element which is very close to an element in $X$ will have the effect that $E_s(X \cup \{ b\} >> E_s(X)$.  (The element $b \in S$ is an 'almost twin' if $\exists x \in X $ with $d(b, x) << 1$.)  
		\item Note that $E_s$ has monotonicity in varieties (though $E_s$ gets worse if you add a point). Note that $-E_s$ has the "opposite" of monotonicity in varieties: $-E_s(A\cup \{ b\}) \leq -E_s(A)$. 
		\item Complexity of the computation of the individual contribution ($\mbox{indD}(b)$) of an element $b \not \in X$  is $\mathcal{O}(2n)=\mathcal{O}(n)$ for you need to compute the sum of $n$ terms of the form $(\frac{1}{d(b,x_i)})^s$ or equal to 0  where $x_i \in X, i=1,\cdots n$, and then compute the sum. Clearly this is an upper as well as a lower bound. 
		
		For the case $x \in X$ see the below.
		
		The computation of the $n$-tuple  $\mbox{allC}(X)$ has complexity $\mathcal{O}(n^2)$. Clearly $\mbox{indC}(x)$ for $x \in X$ is equal to
		$\sum_{y \in X \setminus \{x\}, d(x,y) \neq 0} (\frac{1}{d(x,y})^s$ which takes $\mathcal{O}(2(n-1)) = \mathcal{O}(n)$. Hence 
		the computation of $\mbox{allC}(X)$ has as upper bound $\mathcal{O}(n\cdot n)$. 
		
		A lower bound is also at least $\Omega(n^2)$ because for each   $\mbox{indC}(x), x \in X$ you need for each of  the given distances $d(x,y), y \in X\setminus\{ x\}$ determine whether to compute $(\frac{1}{d(x,y)})^s$ and then sum them up: it takes at least $\Omega(2n)$. Hence for $\mbox{allC}(X)$ it takes at least $\Omega(n\cdot 2n) = \Omega(n^2)$
		\item \emph{Submodularity} does not hold which has been shown in \cite{Falcon-Guillermo2024RieszEnergy}. 
		
		This statement is immediately clear in light of the second characterization of submodularity (Lemma \ref{lemma:submodular}). As the similarity space is non-trivial we can find two disjoint subsets $T_1, T_2$ of $S$ which are non-empty such that there exists a pair of points $t_1 \in T_1$ and $t_2 \in T_2$ with $d(t_1, t_2) > 0$. For these sets the following strict inequality holds:
		$$ E_s(T_1) + E_s(T_2) \color{red} < \color{black} E_s(T_1 \cup T_2) + E_s(\emptyset) $$ 
		as the distances between the two sets are not accounted for in the left hand side of the inequality. 
		
		
		In Proposition \ref{proposition:Es-not-submodular}   we will stick to non-trivial similarity spaces (when the similarity space is a metric space its cardinality will necessarily be bigger than one) 
		for establishing the relevant inequality and characterize when the inequality is an equality. 
		
		To simplify notation we introduce the following notation: 
		Let $A$ and $B$ be subsets of $S$,then $(d(A,B))^{-s} := \{(\frac{1}{d(a,b)})^{-s} \, | \, a \in A, b \in B, d(a,b) \neq 0 \}$ -- in words the s-energy of all non-zero distances from $A$ to $B$. It will have  
		trivial properties such as $(d(A_1 \cup A_2, B ))^{-s} = (d(A_1, B ))^{-s} + (d(A_2, B ))^{-s}$ in case $A_1$ and $A_2$ are disjoint sets and $(d(A,B))^{-s} = (d(B,A))^{-s}$ .
		
		\begin{proposition}
			\label{proposition:Es-not-submodular}
			Let $S$ be a finite similarity space and all its non-diagonal elements of its distance matrix be bigger than zero. 
			Then $\forall T_1 \in \mathcal{P}(S)$ and $\forall T_2 \in \mathcal{P}(S)$ we have 
			$$E_s(T_1) + E_s(T_2) \leq E_s(T_1\cup T_2) + E_s(T_1 \cap T_2),$$
			where equality holds iff $T_1 \subseteq T_2$ or $T_2 \subseteq T_1$. 
		\end{proposition}
		\begin{proof}
			For the proof of this consider the following three cases (both $T_1 \neq \emptyset$ and $T_2 \neq \emptyset$):
			\begin{description}
				\item[Case I] If $T_1 \cap T_2 = \emptyset$, then $E_s(T_1) + E_s(T_2) < E_s(T_1\cup T_2) + E_s(\emptyset)$, 
				\item[Case II] If $T_1 \subseteq T_2$, then $E_s(T_1)+E_s(T_2) = E_s(T_1\cup T_2) + E_s(T_1\cap T_2)$, as $T_1\cup T_2 = T_2$ and $T_1 \cap T_2 = T_1$. 
				\item[Case III] $T_1\setminus T_2 \neq \emptyset$ and $T_2 \setminus T_1 \neq \emptyset$ and $T_1\cap T_2 \neq \emptyset$.\\ In this case we get:
				$$ E_s(T_1) + E_s(T_2) \color{red} < \color{black} E_s(T_1 \cup T_2) + E_s(T_1\cap T_2 )$$ 
				for:
				$E_s(T_1\setminus T_2) + E_s(T_2) \color{red} < \color{black} E_s(T_1 \cup T_2)$ as $T_1\setminus T_2$ and $T_2$ are disjoint and both non-empty.
				
				We now ask the question which  distances you still need to get the $s$-energy of $T_1\cup T_2$ and which distances do you need in order to get $E_s(T_1)$.
				In order to get from $E_s(T_1\setminus T_2) + E_s(T_2)$ to $E_s(T_1\cup T_2)$ you need to apply the s-energy formula to
				a) the distances from points in $T_1\setminus T_2 $ to points in $T_1 \cap T_2)$ and 
				b) the distances from points in  $T_1\setminus T_2$ to points in $T_2\setminus T_1$. Thus
				$$E_s(T_1\setminus T_2) + 
				(d(T_1\setminus T_2, T_1 \cap T_2))^{-s} + 
				(d(T_1 \setminus T_2, T_2\setminus T_1))^{-s} \color{red} = \color{black}
				$$ $$\color{red} = \color{black} E_s(T_1\cup T_2)$$
				Next we see that $E_s(T_1)$ consists of 
				$$E_s(T_1\setminus T_2) + 
				(d(T_1\setminus T_2, T_1 \cap T_2))^{-s}
				+ E_s(T_1\cap T_2) \color{red} = \color{black} E_s(T_1). $$
				From this we get:
				$$ (E_s(T_1 \cup T_2) + E_s(T_1 \cap T_2)) \color{red} - \color{black} (E_s(T_1)+E_s(T_2)) = (d(T_1\setminus T_2, T_2 \setminus T_1))^{-s}. $$ 
				Hence:
				$ E_s(T_1) + E_s(T_2) \, \color{red} < \color{black} E_s(T_1\cup T_2) + E_s(T_1 \cap T_2) $
				as   $(d(T_1\setminus T_2, T_2 \setminus T_1))^{-s} > 0$
				$\square$
			\end{description} 
		\end{proof} 
		\begin{corollary}
			The indicator $-E_s$ is submodular (confirming a result in \cite{Falcon-Guillermo2024RieszEnergy}) and the inequality is a strict inequality unless the subsets involved are comparable with respect to set inclusion. 
		\end{corollary}
		\begin{corollary}
			For finite subsets $T_1, T_2$ of $S$ the formula 
			$$ (E_s(T_1 \cup T_2) + E_s(T_1 \cap T_2)) \color{red} - \color{black} (E_s(T_1)+E_s(T_2)) = (d(T_1\setminus T_2, T_2 \setminus T_1))^{-s}$$ is also correct for Cases I and II ($T_1 \mbox{ and } T_2$ disjoint or $T_1 \subseteq T_2$ or $T_2 \subseteq T_1$.) (We assume that $E_s(\emptyset) := 0$ and (if $A=\emptyset$ or $B=\emptyset$, then  $(d(A, B))^{-s} := 0$)) 
		\end{corollary}
		\begin{corollary}
			if $ Y \subsetneqq Z \subseteq S$ and $x \in S \setminus Z$, then \\
			$E_s(Y\cup \{x \}) - E_s(Y)\,  \mathbf{<} \, \color{black} E_s(Z \cup \{x \}) - E_s(Z)$. (If $Y=\emptyset$, then the inequality also holds, as $E_s(\{x\})=0$ and $E_s(\emptyset)=0$. We get $0 < E_s(Z \cup \{ x\}) - E_s(Z).)$
		\end{corollary}
		A similar result holds for $-E_s$ with the inequality reversed to $>$,  in other words for $-E_s$ we get, when $Y$ is a subset of $Z$ the marginal gain with respect to $Y$ is always bigger or equal  than the marginal gain with respect to $Z$ and it is strictly bigger in case $Y \subsetneqq Z$ -- the latter is 'almost always' the case. This strengthens the submodularity result obtained in \cite{Falcon-Guillermo2024RieszEnergy}.
		\begin{corollary}
			The statement $ E_s(T_1) + E_s(T_2) \, \color{red} < \color{black} E_s(T_1\cup T_2) + E_s(T_1 \cap T_2) $ also holds for similarity spaces where the off-diagonal entries of the distance matrix may be zero provided  $(d(T_1\setminus T_2, T_2 \setminus T_1))^{-s} > 0$ (in other words, $\exists t_1 \in T_1\setminus T_2$ and $\exists t_2 \in T_2\setminus T_1$ such that $d(t_1, t_2) >0$).  
		\end{corollary}
		
		\item  The following theorem establishes a new result on the NP-Hardness of Riesz s-Energy Subset Selection.
		\begin{theorem} \label{thm:riesz-energy-np-for-finite-metric-spaces}
			\textbf{Riesz $s$-Energy Subset Selection is NP-hard in General Metric Spaces}
			Given a set of size $n$ in a metric space with $d_{ij}$ denoting the pairwise distances between the points in this set, the problem of finding the subset of size $k$ that minimizes the Riesz $s$-Energy is NP-hard (where $s > 0$).
		\end{theorem}
		A proof of this theorem can be found in the Appendix. It is also valid for Similarity Spaces.
		\item The indicators $E_s$ and $-E_s$ ($s > 0$) are isometry invariant for the distance matrices of $X$ and $\iota(X)$ are equal and the values of these indicators are solely dependent on the distance matrix. In fact, you can view these indicators as functions from matrices (if you like square matrices with diagonal entries equal to 0) with entries in $\mathbb{R}_{\geq 0}$ to $\mathbb{R}$. See also Proposition \ref{proposition:isometry-invariance} and its Corollary \ref{corollary:indicators-which-are-invariant}
		\item As the function $(\_)^{-s}: \mathbb{R}_{\geq 0} \rightarrow \mathbb{R}$
		is strictly monotone decreasing for $s$ with $s>0$ we get that $-E_s$ has monotonicity and strict monotonicity in distance. The same holds for $E_s$ except that if you increase an entry in the distance matrix $E_s$ will strictly decrease (as you want to minimize $E_s$ the diversity "gets better").
	\end{enumerate}
	\begin{proposition}\label{proposition:isometry-invariance}
		Let $f$ be a function defined on matrices with real number entries such that its domain includes square matrices of any size and whose entries are non-negative real numbers and all the diagonal elements are equal to 0. (Usually one considers functions with codomain the real numbers but other codomains are possible.) 
		A diversity indicator which takes as input the distance matrix (dm) of finite similarity (sub)spaces and outputs $f(\mbox{dm})$ is isometry invariant. 
	\end{proposition}
	\begin{corollary}\label{corollary:indicators-which-are-invariant} 
		Examples of isometry invariant diversity indicators are the s-energy diversity indicator,  the Max-min diversity indicator, the Solow-Polasky indicator, the Sum indicator,  the Weitzman indicator, and the Gap indicators\cite{emmerich2013quality}. (The latter two we did not define.) As a rotation (or translation or reflection or any rigid motion) of Euclidean space  are isometries of Euclidean space it follows, for instance, that $E_s$ (and $-E_s$) are rotation invariant as described and analyzed in 
		\cite{Falcon-Guillermo2024RieszEnergy} or, for instance, 
		invariant under isometries of Euclidean spaces with respect to the $L^p$ norm $p \geq 1$. 
	\end{corollary}
	
	\begin{remark}
		The concept of isometry invariance is  
		a generalization of the
		concept of rotation invariance as described and analyzed  in \cite{Falcon-Guillermo2024RieszEnergy}. 
		Rotation invariance for Euclidean spaces (described and analyzed in \cite{Falcon-Guillermo2024RieszEnergy} is a special case of  isometry invariance in metric spaces (or similarity spaces).
		This means that isometry invariance implies rotation invariance. 
		The approach in this paper shows the invariance without any computation, also for the rotation invariance of $E_s$ and $-E_s$.  Moreover Proposition \ref{proposition:isometry-invariance} shows that all the indicators mentioned in Corollary \ref{corollary:indicators-which-are-invariant}\, and for $-E_s$ and $E_s$, among others,  are rotation invariant as described in \cite{Falcon-Guillermo2024RieszEnergy} and for any rigid motion of Euclidean space. 
	\end{remark}
	\subsubsection{Theoretical properties of the $D_{\mbox{max-min}}$ indicator}

	\begin{enumerate}
		\item The individual contribution  computation 
		$$(\mbox{indC}(\tilde{x})) := D_{\mbox{max-min}}(X\cup \{ \tilde{x} \} -D_{\mbox{max-min}}(X\setminus  \{ \tilde{x} \})  $$ is $\mathcal{O}(n^2)$ -- this is also a lower bound. The enumeration of all values in the similarity matrix gives rise to the upper bound \( O(n^2) \). 
		\begin{theorem} \label{thm:maxminsubsetselection}
			An upper bound for calculating all single-point contributions of $X$,  $\mbox{allC}(X)$,  of the Max-min diversity is $O(n^3)$, and $O(n^2)$ when excluding specific degenerate distance matrices with duplicate pairwise distances.
		\end{theorem}
		A proof of this theorem can be found in the Appendix. It is also valid for Similarity Spaces. 
		\item Max-min Diversity satisfies twinning: suppose $b$ is a duplicate with respect $X$ then in the distance matrix of $X\cup {b}$ no "new" distances will appear, thus the minimum of the off-diagonal entries of $\mbox{dm}(X)$ and of $\mbox{dm}(X\cup {b})$ is the same. 
		\item Submodularity for Max-min Diversity does not hold: (use $f$ is submodular iff $\forall T_1, T_2$ finite: $f(T_1) + f(T_2) \geq f(T_1 \cup T_2) + f(T_1 \cap T_2)$: see for a counterexample Figure \ref{fig:counterexample-Dmaxmin-submodular}, also $-D_{\mbox{max-min}}$ is not submodular.
		\item The indicator $D_{\mbox{max-min}}$ has monotonicity in distance but does not have strict monotonicity in distance. It is clear that by enlarging distances between points the value of this indicator will not be reduced but it can happen that increasing distances will not increase the value of this indicator: see for a counterexample Figure \ref{fig:counterintuitive}.   
		\item The Max-min Diversity is isometry invariant (see Proposition \ref{proposition:isometry-invariance}).
		\item The subset selection for the Max-min Diversity indicator is known to be NP-hard\cite{agarwal2006computing}.
		\begin{figure}[ht]
			\centering
			\begin{minipage}{0.5\textwidth} 
				\centering
				\begin{tikzpicture}[scale=0.5]
					\foreach \x in {0,8}{
						\draw[thick, blue] (\x,0) -- (\x,4);
					};
					\foreach \y in {0,4}{
						\draw[thin, blue, dashed] (0,\y) -- (8,\y);
					};
					\draw[thick, blue] (0,2) -- (8,2);
					\draw[thin, blue, dashed] (0,0) -- (8,2);
					\draw[thin, blue, dashed] (0,0) -- (8,4);
					\draw[thin, blue, dashed] (0,2) -- (8,0);
					\draw[thin, blue, dashed] (0,2) -- (8,4);
					\draw[thin, blue, dashed] (0,4) -- (8,0);
					\draw[thin, blue, dashed] (0,4) -- (8,2);
					
					\foreach \y/\name in {0/a, 2/b, 4/c} {
						\fill[red] (0,\y) circle (2pt);
						\node[left] at (0,\y){\textbf{\name}}; 
					}; 
					\foreach \y/\name in {0/d, 2/e, 4/f} {
						\fill[black] (8,\y) circle (2pt);
						\node[right] at (8,\y){\textbf{\name}}; 
					};
					\draw[thick, black] (0,2) circle (3pt);
					\draw[thick, red] (8,2) circle (3pt);
					
					\foreach \y/\name in {1/1, 3/1} {
						\node[left] at (0,\y){\textbf{\name}}; 
						\node[right] at (8,\y){\textbf{\name}};
					};
					\node[above] at (4,4){\textbf{4}};
					\node[above] at (4,0){\textbf{4}};
					\node[above] at (7.0, 2.3){\textbf{$\sqrt{17}$}};
					\node[above] at (7.0, 0.5) {\textbf{$\sqrt{20}$}};
				\end{tikzpicture}
			\end{minipage}%
			\hfill
			\begin{minipage}{0.45\textwidth} 
				\centering
				{\footnotesize
					\begin{tabular}{|l|l|l|l|l|l|l|}
						\hline
						& a & b & c & d & e & f \\
						\hline
						a & 0 & 1 & 2 & 4 & $\sqrt{17}$ & $\sqrt{20}$ \\ 
						\hline
						b &  & 0 & 1 & $\sqrt{17}$ & 4 & $\sqrt{17}$ \\ 
						\hline
						c &  &   & 0 & $\sqrt{20}$ & $\sqrt{17}$  & 4 \\ 
						\hline
						d &  &   &  & 0 & 1  & 2 \\ 
						\hline
						e &  &   &  &  & 0  & 1 \\ 
						\hline
						f &  &   &  &  &    & 0 \\ 
						\hline
					\end{tabular}
				}
			\end{minipage}
			\caption{Counterexample: The $D_{\mbox{max-min}}$ Diversity Indicator is not Submodular. Let $X = \{a,b,c,d,e,f \}$. 
				$S = \{ a,b,c, e \}$ (the red filled circles or red circle) and $T = \{ d,e,f, b \}$ (the black filled circles and the black circle). Then $D_{\mbox{max-min}}(S) + D_{\mbox{max-min}}(T) < D_{\mbox{max-min}}(S\cup T) +D_{\mbox{max-min}}(S \cap T)$, 
				for $D_{\mbox{max-min}}(S) = 1, D_{\mbox{max-min}}(T) = 1, D_{\mbox{max-min}}(S\cup T) = 1, \mbox{ and } D_{\mbox{max-min}}(S\cap T) = 4$ ( see the Distance Matrix of $X$ in Table on the right).  
				Notice that you can create counterexamples for any finite sizes of the sets involved. Also you can create counterexamples which have an essential Euclidean dimension of arbitrary largeness, that is, $X \subset \mathbb{R}^s$ such that the convex hull of $X$ has dimension $s$. Notice also that you can create counterexamples which will show that $-D_{\mbox{max-min}}$ is not submodular.}
			\label{fig:counterexample-Dmaxmin-submodular}
		\end{figure}
	\end{enumerate}
	%
	
	\subsubsection{Theoretical Properties the Solow Polasky indicator}
	
	\begin{enumerate}
		\item Preliminary remark. The $\mbox{sim}(X, \theta)$ is a strictly positive definite matrix for all $\theta > 0$ in case $S \subseteq \mathbb{R}^n$ and $X$ is a finite subset of $S$ and also ultra metric spaces are positive definite (see \cite{leinster2013magnitudeOfMetricSpaces} ). Also for any metric space of size 2 the $\mbox{sim}$ is strictly positive definite (either use the Sylvester characterization (and the off-diagonal elements are in (0,1) ) or straightforward: show for all $v \neq 0$ and for all $\theta > 0$:  $v^{\top}\mbox{sim}\,v > 0$ using completing of the square. For metric spaces of size 3 the same holds for $\theta >0$ using Sylvester characterization and using the triangle inequality. Metric spaces of size 4 are also positive definite as is proved by \cite{meckes2013positiveDefiniteMetricSpaces}. Spaces of size 5 are not necessarily postive definite \cite{leinster2013magnitudeOfMetricSpaces}. 
		\item The indicator SP has the twinning property. 
		Let $b$ be a duplicate for $X$ and assume $X$ has magnitude then $X \cup \{b\}$ has magnitude and the magnitudes are equal. 
		For let 
		$$
		\mbox{sim}(X, \theta) =
		\begin{bmatrix}
			\color{blue}1 & \color{blue}e_{1,2} & \color{blue}\cdots & \color{blue}e_{1, n-1} & \color{blue}e_{1, n}\\
			\color{blue}e_{2,1} & \color{blue}1 & \color{blue}\cdots & \color{blue}e_{2, n-1} & \color{blue}e_{2, n} \\
			\color{blue}\vdots & \color{blue}\color{blue}\color{blue}\vdots&\color{blue} \ddots & \color{blue}\vdots &\color{blue} \vdots \\
			\color{blue}e_{n,1} & \color{blue}e_{n,2} & \color{blue}\cdots  & \color{blue}e_{n,n-1} &\color{blue}1
		\end{bmatrix}, 
		$$
		be the symmetric similarity matrix, where $e_{i,j} := \exp(-\theta \cdot d_{i,j})$ and $d_{i, j}$ are the distance entries of $\mbox{dm}(X)$. And let $[w_1,\cdots, w_n]^{\top}$ be a weighting of  $\mbox{sim}(X, \theta)$, that is  
		$$
		\mbox{sim}(X, \theta) 
		\begin{bmatrix}
			w_1 \\
			w_2 \\
			\vdots \\
			w_{n} \\
		\end{bmatrix}
		=
		\begin{bmatrix}
			\color{blue}1 & \color{blue}e_{1,2} & \color{blue}\cdots & \color{blue}e_{1, n-1} & \color{blue}e_{1, n}\\
			\color{blue}e_{2,1} & \color{blue}1 & \color{blue}\cdots & \color{blue}e_{2, n-1} & \color{blue}e_{2, n} \\
			\color{blue}\vdots & \color{blue}\color{blue}\color{blue}\vdots&\color{blue} \ddots & \color{blue}\vdots &\color{blue} \vdots \\
			\color{blue}e_{n,1} & \color{blue}e_{n,2} & \color{blue}\cdots  & \color{blue}e_{n,n-1} &\color{blue}1 \\
		\end{bmatrix}
		\begin{bmatrix}
			w_1 \\
			w_2 \\
			\vdots \\
			w_{n} \\
		\end{bmatrix}
		= 
		\begin{bmatrix}
			1 \\
			1 \\
			\vdots \\
			1 \\
		\end{bmatrix}
		$$
		As $b$ is a duplicate for $X$ we get that similarity matrix of $X\cup \{ b\}$ is as follows (assume that twin of $b$ in $X$ is represented in $\mbox{sim}(X,\theta)$ as the last column (which is equal to the transpose of the last row)).
		$$
		\mbox{sim}(X\cup \{b \}, \theta) =
		\begin{bmatrix}
			\color{blue}1 & \color{blue}e_{1,2} & \color{blue}\cdots & \color{blue}e_{1, n-1} & \color{blue}e_{1, n}& e_{1,n}\\
			\color{blue}e_{2,1} & \color{blue}1 & \color{blue}\cdots & \color{blue}e_{2, n-1} & \color{blue}e_{2, n} & e_{2,n} \\
			\color{blue}\vdots & \color{blue}\color{blue}\color{blue}\vdots&\color{blue} \ddots & \color{blue}\vdots &\color{blue} \vdots & \vdots  \\
			\color{blue}e_{n,1} & \color{blue}e_{n,2} & \color{blue}\cdots  & \color{blue}e_{n,n-1} &\color{blue}1 & 1 \\
			e_{n,1} & e_{n,2} & \cdots  & e_{n,n-1} &1 & 1 \\
		\end{bmatrix}, 
		$$
		By inspection we see that $\mbox{sim}(X\cup \{b \})$ has the following weightings ($\alpha, \beta \in \mathbb{R}$ such that $\alpha + \beta = 1$):
		
		$$
		\mbox{sim}(X\cup \{b \}, \theta) 
		\begin{bmatrix}
			w_1 \\
			w_2 \\
			\vdots \\
			\alpha \cdot w_{n} \\
			\beta \cdot w_{n}
		\end{bmatrix}
		=
		\begin{bmatrix}
			\color{blue}1 & \color{blue}e_{1,2} & \color{blue}\cdots & \color{blue}e_{1, n-1} & \color{blue}e_{1, n}& e_{1,n}\\
			\color{blue}e_{2,1} & \color{blue}1 & \color{blue}\cdots & \color{blue}e_{2, n-1} & \color{blue}e_{2, n} & e_{2,n} \\
			\color{blue}\vdots & \color{blue}\color{blue}\color{blue}\vdots&\color{blue} \ddots & \color{blue}\vdots &\color{blue} \vdots & \vdots  \\
			\color{blue}e_{n,1} & \color{blue}e_{n,2} & \color{blue}\cdots  & \color{blue}e_{n,n-1} &\color{blue}1 & 1 \\
			e_{n,1} & e_{n,2} & \cdots  & e_{n,n-1} &1 & 1 \\
		\end{bmatrix}
		\begin{bmatrix}
			w_1 \\
			w_2 \\
			\vdots \\
			\alpha \cdot w_{n} \\
			\beta \cdot w_{n}
		\end{bmatrix}
		= 
		\begin{bmatrix}
			1 \\
			1 \\
			\vdots \\
			1 \\
			1
		\end{bmatrix}
		$$
		And the sum of a weighting in the extended matrix is the same as the sum of the weighting of the original matrix. 
		\item The SP indicator is not submodular. The following example which was communicated to us by Steve Huntsman shows that it is in general not submodular: Consider the following four points in $\mathbb{R}^2$: $a=(-1,0), b=(1,0), c=(2,0), d=(0,1)$ then 
		$$( \mbox{SP}(\{a, b, d \}) - \mbox{SP}(\{a,b\}) )
		-( \mbox{SP}(\{a, b, c,  d \}) - \mbox{SP}(\{a,b,c\}) ) = \color{red}\mathbf{-}\color{black}0.0144346.$$ 
		This number would have been non-negative in case of submodularity. 
		\item Individual contribution for SP is polynomial $\mathcal{O}(n^3) \mbox{ or } \mathcal{O}(n^{2.8})$ -- or $\mathcal{O}(n)$ in case one of the terms in $\mbox{SP}(X\cup \{ x \}) - \mbox{SP}(X)$ is known. An upper bound for the computation of $\mbox{allC}(X)$ is $\mathcal{O}(n^{2.8} + n^2)= \mathcal{O}(n^{2.8})$. A lower bound is the same as the lower bound for matrix inversion. (cf. \cite{ulrich2011diversity})
	\end{enumerate}
	\begin{table}[ht] 
		\caption{Diversity Indicators and Their Properties}
		\label{table:propertiesOfDivIndicators}
		\small
		\centering
		\begin{tabular}{|l|l|l|l|l|l|l|}
			\hline
			Indicators                   & 
			\begin{sideways} Monotonicity in varieties \end{sideways} &
			\begin{sideways}Twinning\end{sideways} &
			\begin{sideways}Monotonicity in Distance \end{sideways} &
			\begin{sideways}Strict Monotonicity in Distance \end{sideways} &
			\begin{sideways}Uniformity \end{sideways} & 
			\begin{sideways}Computational Effort \end{sideways} \\
			\hline
			\hline
			Riesz s-Energy            &Y&N&Y&Y&Y&$\mathcal{O}(\frac{n(n-1)}{2})$\\ 
			\hline
			Max-min Diversity         &N&Y&Y&N&N&$\mathcal{O}(n^2)$\\ 
			\hline
			Solow-Polasky        &Y&Y&Y&N&Y&$\mathcal{O}(n^{2.8})$\\
			\hline
		\end{tabular}
		
		\vspace{0.5cm} 
		
		\begin{tabular}{|l|l|l|l|l|l|}
			\hline
			Indicators                   & 
			\begin{sideways} single point contribution \end{sideways} &
			\begin{sideways} all contributions $\mbox{allC}(X)$ \end{sideways} &
			\begin{sideways} subset selection \end{sideways} &
			\begin{sideways} submodularity \end{sideways} &
			\begin{sideways} isometry invariant \end{sideways} \\
			\hline
			\hline
			Riesz s-Energy            &$\mathcal{O}(n)$& $\mathcal{O}(n^2)$&NP hard& Y for $-E_s$ & Y\\ 
			\hline
			Max-min Diversity         &$\mathcal{O}(n^2)$&$\mathcal{O}(n^3)$&NP hard&N&Y\\ 
			\hline
			Solow-Polasky        &$\mathcal{O}(n^{2.8})$&$\mathcal{O}(n^{2.8})$&??&N&Y\\
			\hline
		\end{tabular}
		
	\end{table}
	
		
		
		
		\section{Case study}
		\label{sec:noah}
		In our paper, we took as a basis the NOAH method formulated in \cite{ulrich2011diversity}. The idea of the method is to combine optimization (minimization of objective functions) and increase diversity in the sample. The NOAH algorithm consists of three components at each iteration: conventional multi-objective optimization, lowering the so-called barrier (the values of the objective function below which solutions are considered acceptable), and optimization of diversity. For detailed information, refer to \cite{ulrich2011diversity}.
		
		As the optimization algorithm we applied the classic NSGA-II algorithm \cite{deb2002fast}. The original paper uses the SP diversity measure and one objective function for optimization. We expanded the method to two objective functions and decided to test other popular diversity indicators, Max-min diversity and s-energy, without radically changing the idea of the NOAH algorithm. All three indicators have their own subset selection procedures incorporated in the NOAH algorithm: SP as in \cite{ulrich2011diversity}, Max-min as in \cite{porumbel2011simple}, s-Energy as in \cite{Falcon-Guillermo2024RieszEnergy}. Multiobjective optimization was carried out according to the following principle. At every moment when a single objective function in the algorithm was compared with the value of the barrier $b$, in our case, the vector of objective functions was compared component-by-component with the vector $b$. At the second stage (lowering barrier $b$), we reduced one of the components of vector $b$ randomly selected. 
		
		We made a small change for the NOAH algorithm. At each iteration where diversity optimization is applied, within the algorithm, when checking whether a new offspring is suitable, we added an additional condition that the diversity does not decrease in the new generation when this descendant is added (in the case of $s-$Energy, it does not increase). Adding any new element increases the $s-$Energy of the set. Therefore, in this case, the $s-$Energy of the population was compared at the stage before the formation of this offspring and after replacing the mutated element with a potential offspring (instead of simply adding it). Since the algorithm continues to search for a suitable offspring in each generation until it finds one, and our computing power was limited, we added the maximum number of mutation attempts. If the counter reaches maximum attempts to apply the mutation to different objects in the population, but still does not meet the conditions, a new offspring is not formed. 
		
		As a sample, we used 20 points randomly selected inside the box $\{(x,y), x \in [0, box\_size], y \in [0, box\_size]\}$. We restricted the search space to the box because otherwise the ideal outcome for any diversity optimization algorithm would be an infinite distance of points from each other on the plane. 
		
		\begin{figure}[t]%
			\centering \subfloat{\includegraphics[height=4cm]{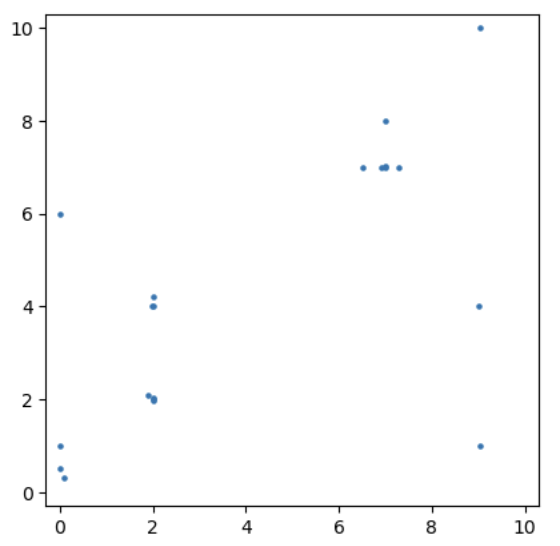} }%
			\subfloat{\includegraphics[height=4cm]{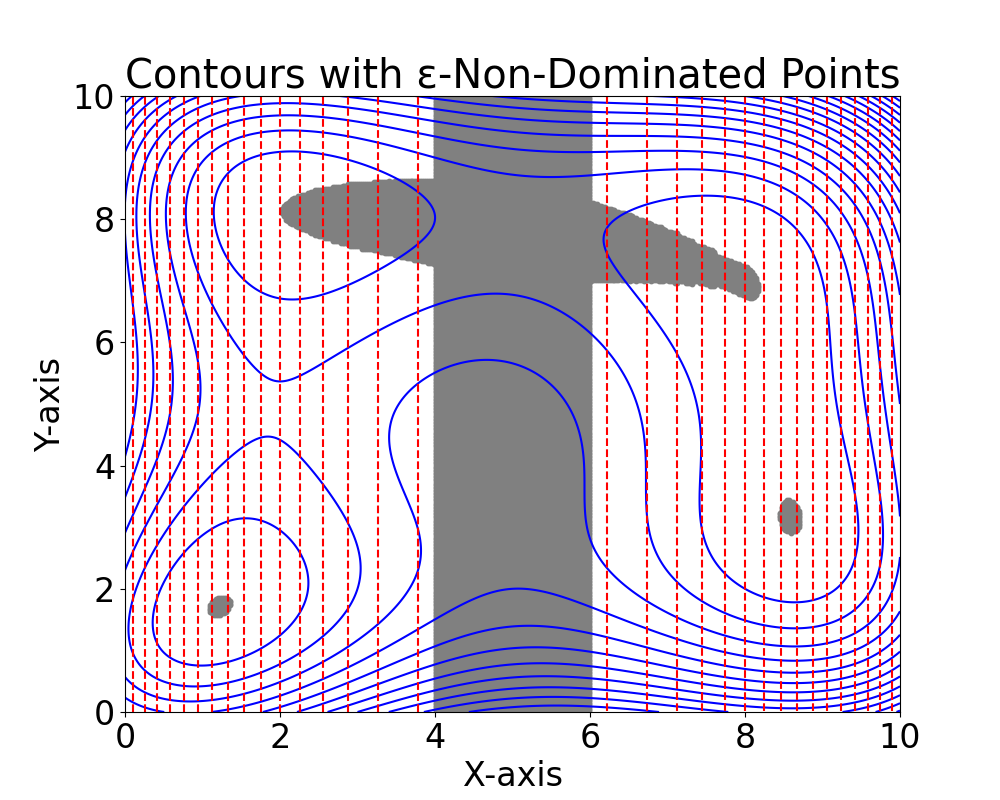}}
			\caption{\label{fig:sampleeff} Sample of 20 points distributed in the box $\{(x,y), x \in [0, 10], y \in [0, 10]\}$ (left). Obtained efficient set and contours of shifted Himmelblau and Paraboloid objective functions for the biobjective problem with $\epsilon$-dominance, $\epsilon = 1$ (right).}
		\end{figure}
		
		
		The mutation operator was used as follows: given a point in the population, we choose a random angle and move in a given direction by the the distance defined as mutation rate (hyperparameter) multiplied by a random number between 0 and 1. This ensures a uniform distribution of the jump of points on the plane inside the circle with a radius of mutation rate and the center in the mutated point. Since we want to stay inside the box in case the mutation produces a sample outside the box, we used the reflection of the point relative to the sides of the box. To keep the experimental setup simple crossover was not applied. In all experiments, we used $box\_size$ equal to 10 and mutation rate equal to 10, which provides average step of length 5 in the randomly chosen direction.
		
		The NOAH algorithm uses a stopping criterion other than a certain generation number, at which the algorithm stops running. Instead, it uses a parameter $c$ that defines after how many generations of no improvement of diversity measure we stop the diversity optimization. For the main experiment, we set $c=3$. 
		Since the main idea of the work is to compare different diversity measures with each other when applied in the NOAH algorithm, in each of the experiments we considered all three diversity measures on populations. 
		
		The multiobjective optimization problem used a shifted Himmelblau function $f_1(x, y) = ((x-5)^2 + (y-5) - 11)^2 + (x-5 + (y-5)^2 - 7)^2$, a classical two-variable problem with four local optima. We shifted it to keep its optima inside the box. The second objective function $f_2$ was a Paraboloid function $f_2(x, y) = (x-5)^2$. Both functions were minimized.
		
		The efficient set for this two-objective optimization problem consists of non-dominated points with additive $\epsilon-$ dominance. Point $(x_1, y_1)$ $\epsilon$-dominates point $(x_2, y_2)$ if $(f_1(x_1, y_1) \leq f_1(x_1, y_1) - \epsilon)\text{ and }  (f_2(x_2, y_2) \leq f_2(x_2, y_2) - \epsilon)$. The value of $\epsilon$ was taken equal to 1. We used a grid of size $100 \times 100$ to approximate and visualize the efficient set. The illustration of the efficient set in Figure \ref{fig:sampleeff} (right), reveals that with the chosen value of $\epsilon=1$, we can still observe the four optima typical of the Himmelblau function. However, the addition of the Paraboloid objective causes two of these optima to become interconnected, while also introducing a new connected region in the center of the plot. This analysis aimed to identify the algorithm's potential convergence set and to use it as a reference for calculating the Hausdorff distance, evaluating the algorithms' effectiveness in capturing all efficient regions.
		
		All three NOAH algorithm variants, each based on a different diversity indicator, exhibited similar behavior across their iterations. A typical example of the algorithm’s run is shown in Figure \ref{fig:typrun}. The NOAH optimization process alternates between convergent and explorative phases, optimizing objectives and then diversity in a cycle. Including diversity optimization helps find new efficient areas, but over time, as a result of the convergent phases, the diversity reduces. For detailed plots on all three NOAH versions, see supplementary material \cite{pereverdievacode}.
		
		\begin{figure}[t] 
			\centering \includegraphics[width=\linewidth]{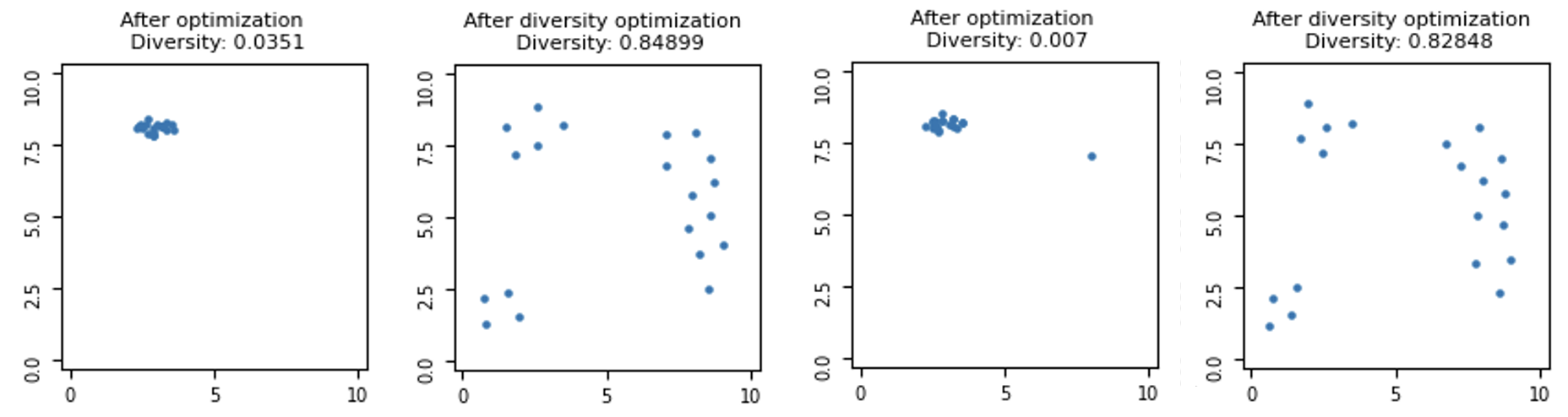} \caption{Typical run. Iterations 6 and 7 of the NOAH algorithm with the Max-min diversity indicator. From left to right: normal optimization and Max-min diversity optimization for iteration 6, followed by normal optimization and Max-min diversity optimization for iteration 7.} \label{fig:typrun} 
		\end{figure}

		For statistically reliable outcomes, each algorithm was executed 30 times, and the results were averaged across these runs. Figure \ref{fig:resal} presents the average values for each diversity measure and the objectives, with the standard error of the mean represented by dashed lines. Note that some values have been scaled for the sake of clarity.
		As mentioned earlier, all three diversity measures decline over time: Max-min and SP values decrease, while Riesz s-Energy increases. This is expected, as the algorithm primarily focuses on optimizing objectives, with the added goal of exploring the solution space to find diverse solutions during the run. 
		
		\begin{figure}[t] \centering \includegraphics[width=\linewidth]{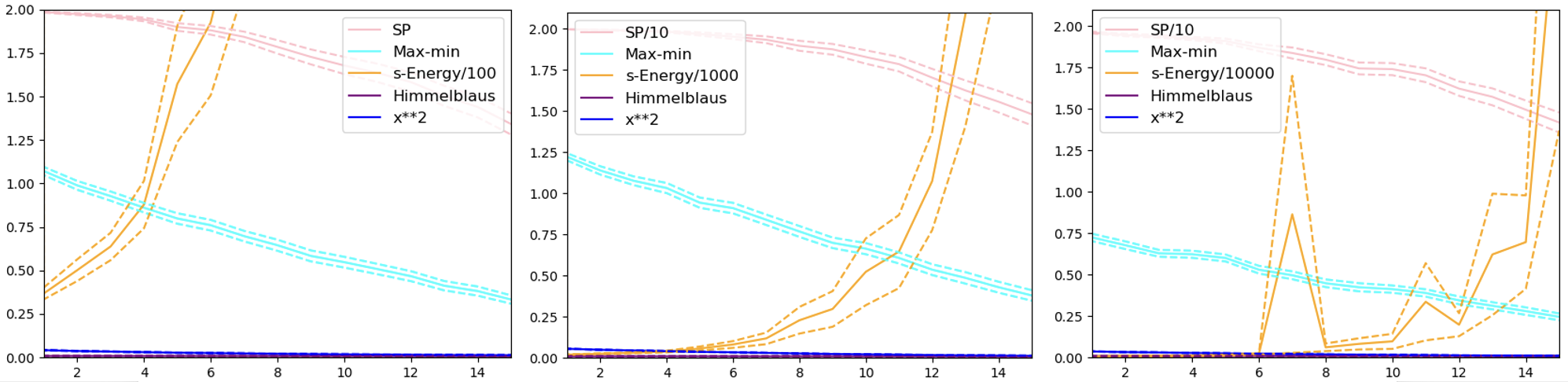} \caption{Diversity measures averaged over 30 runs for each algorithm. From left to right: NOAH with Max-min optimization, s-Energy optimization and SP optimization. Solid lines represent the mean over 30 runs, with dashed lines indicating means $\pm$ the standard error of the mean.} \label{fig:resal} \end{figure}
		
		Instead of comparing the diversity measures individually for each algorithm, we calculated each diversity measure across all three algorithms and plotted them together, as shown in Figure \ref{fig:resmeas}. Additionally, we used the efficient set from Figure \ref{fig:sampleeff} (right) as a reference to compare the sample across iterations, calculating the Hausdorff distance between the efficient set and the sample, and plotted the average distance per iteration across 30 runs.
		
		\begin{figure}[t] \centering \includegraphics[width=\linewidth]{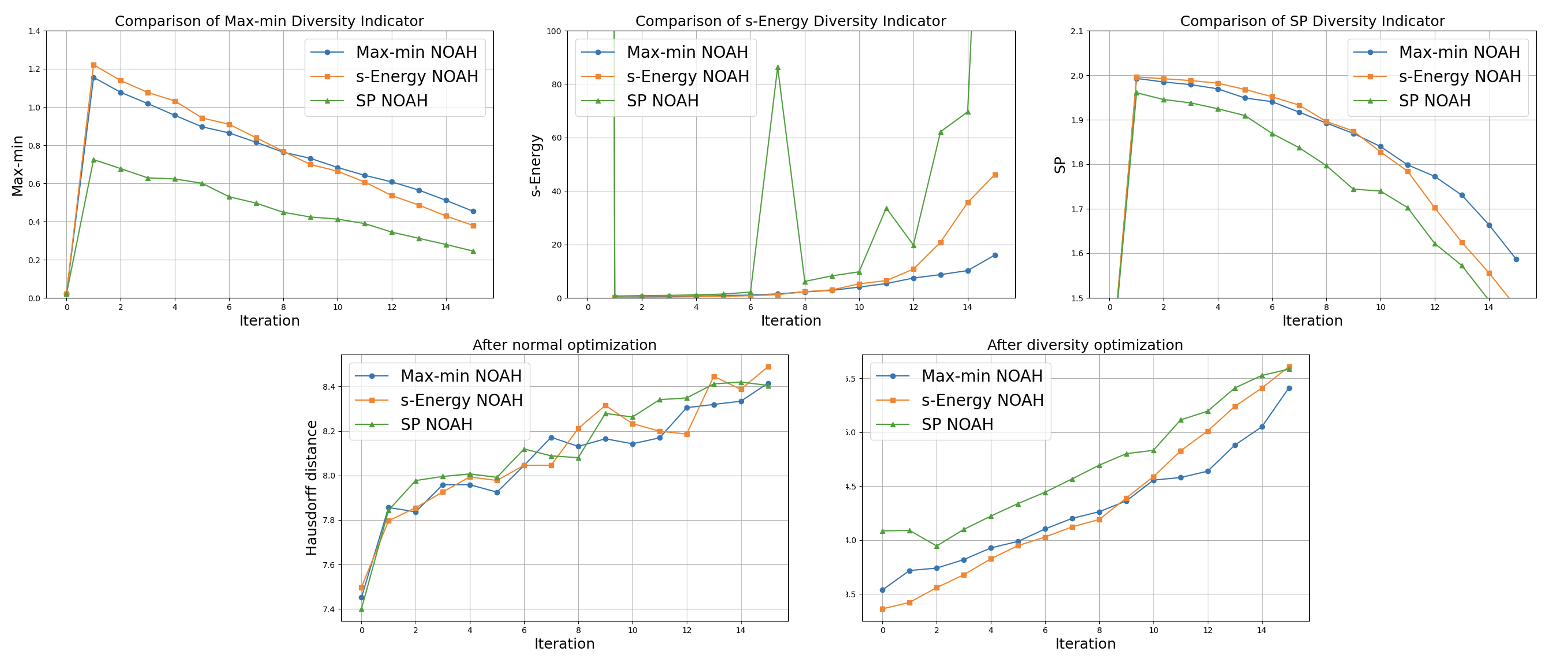} \caption{Diversity measures calculated during the execution of the three algorithms, and the Hausdorff distance calculated at two stages of each iteration: after objective optimization and after diversity optimization.} \label{fig:resmeas} \end{figure}
		Despite the fact that all three diversity measures worsen over time, Max-min optimization shows the slowest rate of deterioration across all three measures. Moreover, from Figure \ref{fig:resmeas}, we observe that the Hausdorff distance after normal objective optimization remains similar across the algorithms, suggesting that optimization negates much of the diversity's impact at round. However, differences emerge during the subsequent diversity optimization stage: Max-min and s-Energy optimizations perform better than SP optimization. From Figure \ref{fig:resmeas} it is clear that Riesz s-Energy and Max-Min Diversity are correlated.
		
		From the Figure \ref{fig:resmeas}, it is clear that despite its theoretical appeal SP optimization performed the worst, even in terms of the SP diversity measure itself. This algorithm also resulted in the largest Hausdorff distance to the reference set. Riesz s-Energy is sensitive to outliers due to its definition, and both Max-min and Riesz s-Energy optimizations managed to stabilize it across 30 runs, while SP optimization exhibited high peaks in s-energy (see Figure \ref{fig:resmeas} (upper middle)).
		
		
		In addition, in the case of Max-min diversity optimization, the following change to the algorithm was proposed. At each step of diversity optimization, when generating offspring, only elements that contribute non-zero values to the overall diversity metric should be considered for selection as parents. If such elements exist, they should be chosen with a positive probability, denoted by $prob$, and thereby influence the minimum distance in the set. To test Max-min diversity optimization, we used NOAH's diversity optimization with: 1000 repeats, b = [3, 3], c = 20, mutation rate = 10, and parent selection probability = $0.9$. For the experiment we measured the number of generations required to achieve a Max-min diversity value of 1.9 for 1000 independent runs of both algorithms. We compared the performance before and after the change using a t-test for two independent samples.
		
		In the modified algorithm, the average number of generations was 31.84 with a standard deviation of 4.44, compared to 32.24 and 4.13 for the original.
		A two-sample t-test (for independent samples) was performed to assess the statistical significance of this difference. The t-statistic was -2.065, and the corresponding p-value was 0.0391. Given that the p-value is below the conventional threshold of 0.05, the difference in the number of generations between the two algorithms is statistically significant, indicating that the modification provides a slight but meaningful improvement in efficiency.

		\section{Discussion and Conclusion}
		\label{sec:discussion_and_outlook}
		
		
		In this paper, we have compiled both well-established theoretical and computational properties, proved new theorems, and highlighted remaining open questions. A central result of our work is presented in Table \ref{table:propertiesOfDivIndicators}, which summarizes these properties. 
		In the case study, Max-min diversity and Riesz s-Energy showed a correlation and both surpassed SP diversity metric in all aspects we considered, including comparison with the reference efficient set, that is, in the ability to explore various optima of the problem Max-min not only converges well but also scales efficiently, making it a reliable choice among the three key diversity metrics. Although the SP diversity metric holds theoretical appeal, it has not shown strong performance in empirical studies. Nevertheless, it continues to be widely used as a standard in diversity optimization.
		
		The study shows NOAH optimization alternates between objectives and diversity, suggesting hyperparameter tuning to avoid over-emphasis on one side at the cost of the other, though it might harm optimization. Evolutionary Level Set Approximation \cite{emmerich2013quality} could be considered as an alternative to NOAH to improve diversity. Our analysis indicates diversity optimization benefits from focusing on diversity indicators, affecting selection and mutation in algorithms. We modified Max-min optimization for faster convergence.
		
		
		Future research should include systematic benchmarks with a range of problems and algorithms and hyperparameters. Applying these results to real-world scenarios, such as Building Spatial Design \cite{pereverdieva2023prism} and Molecular Compound Optimization \cite{bender2004molecular}, where only a dissimilarity function exists, would also be interesting. Although dissimilarity measures have been discussed in SP and Quality-Diversity Optimization \cite{steveHuntsman2023diversityDissimilarity}, broader coverage of diversity indicators is needed.
		
		\paragraph*{\bf Appendix}

		To prove Theorem \ref{thm:riesz-energy-np-for-finite-metric-spaces}, we show k-Clique polynomially reduces to Riesz $s$-Energy Subset Selection:
		\begin{definition}
			\textbf{k-Clique Problem:}
			Given an undirected graph and an integer, the k-Clique problem asks if there is a subset of k vertices where every pair is connected by an edge.
		\end{definition}
		\begin{lemma} \label{lemma:metricFromGraph}
			Let $G=(V,E)$ be an undirected graph. Let $\Delta$ denote the diagonal of $V\times V$ (that is, $\Delta = \{ (v,v)\, |\,  v \in V\}$) and let  $d: V\times V \longrightarrow \mathbb{R}_{\geq 0}$ be defined as follows.  
			$$
			d(i,j) := 
			\begin{cases}
				2, & if (i,j) \in V \times V \setminus  \Delta \mbox{ and } (i,j) \in E \\
				1, & if (i,j) \in V \times V \setminus  \Delta \mbox{ and }  (i,j) \not \in E \\
				0, & otherwise
			\end{cases}
			$$
			Then $(V, d)$ is a metric space, i.e. $d$ is a metric. 
		\end{lemma}
		\begin{proof} Clearly $d$ is non-negative, symmetric, and $d(x,y) = 0 \mbox{ iff } x = y$. Moreover it is easy to see that $d$ satisfies the triangle inequality. 
		\end{proof}
		
		\begin{proof} of Theorem \ref{thm:riesz-energy-np-for-finite-metric-spaces}.

			It is easy to prove this theorem  using Lemma \ref{lemma:metricFromGraph}. For the proof we assume $s > 0$. 
			Given a (finite) graph $G=(V,E)$ and $k \leq |V|$ we can determine whether this graph has a clique of size $k$ (and also provide a witness k-clique) by translating this problem in polynomial time into a Riesz k-subselection problem for the associated metric space described in Lemma \ref{lemma:metricFromGraph}. For any subset $V'$ of $V$ of size $k$ we have: 
			$$
			\frac{k(k-1)}{2^s} \leq \sum_{(i,j) \in V'\times V' \setminus \Delta} \frac{1}{(d_{i,j})^s} \leq \frac{k(k-1)}{1^s} = k (k-1),
			$$
			where the middle term is the Riesz s-energy of $V'$. Clearly, $V'$ is a k-clique iff the Riesz s-energy of $V'$ attains the lower bound (and dually: $V'$ is an independent subset of the graph iff the Riesz s-energy of $V'$ attains the upper bound). 
			The NP-complete clique problem can thus be recast in polynomial time to an instance of ESS. Note that theorem is valid $\forall s >0.$ A slight modification of the proof will show that the theorem is valid $\forall s \neq 0$. 
		\end{proof}
		
		
		\begin{proof} of Theorem \ref{thm:maxminsubsetselection}.
			
			%
			
			Firstly, we calculate the diversity of the set and the distance matrix -- both take $O(n^2)$. 
			As we define the diversity minimum among all pairwise distances, contribution of most elements will be equal to $0$. Say the smallest distance is $d_{min}$. Then, there are three cases to consider. 
			
			\textbf{Case 1}: The deletion of a single element never changes $D_{\mbox{max-min}}$ we obtain 
			$indD(a_i)=0, \ \forall i \in \{1, 2, \dots, n\}. $ This case is obtained when there are three different elements $\{a_i, a_j, a_k\}, \  i, j, k \in \{1, 2, \dots, n\}$ with $d(a_i, a_j)$ $ = d(a_j, a_k) = d(a_k, a_i) = d_{min}$, or when there are at least two non-overlapping pairs of elements with the smallest distance: $d(a_i, a_j) = d(a_k, a_m) = d_{min}$. The second step adds to the time complexity time it takes to identify this case. To check it we need to create a list of pair of elements which have distance equal to minimal distance, it takes $O(n^2)$.For each pair of points, we calculate two possible points to form a regular triangle, which takes $O(1)$. Then, checking for points with matching coordinates takes $O(n)$ at worst, leading to $O(n^3)$. Identifying two non-overlapping pairs at the minimal distance takes $O(n^2)$, given there are $O(n^2)$ possible pairs.
			
			\textbf{Case 2:} multiple elements have distance $d_{min}$ to a central element (star pattern). This central element contributes non-zero $0$, while others contribute $0$. Determining this configuration requires $O(n^2)$ as each element must have at least two others at distance $d_{min}$. The contribution of the central element is found by identifying the smallest element in the distance matrix after removing its row and column $d'_{min}$, which takes $O(n^2)$. Thus, the complexity is $O(n^2)$ and the contribution of the element is $d'_{min} - d_{min}$.
			%
			
			\textbf{Case 3}:  There is just one pair with distance equal to $d_{min}$, the contribution of all other elements is $0$. All we need is to calculate the contribution of these two elements. For both elements, we proceed as follows: delete the row and the column of matrix $M$ corresponding to the element and find the minimal element. It takes $O(n^2)$ by analogy with the previous case. 
			
			The time complexity of the procedure outlined in the proof reduces from $O(n^3)$ to  $O(n^2)$  if neither condition of case 1 is met. Point sets that are randomly drawn from a continuous space almost never satisfy these conditions, which means that Case 1 can be safely ignored, leading to a time complexity $O(n^2)$. Symbolic perturbation could be considered to avoid degenerate cases \cite{devillers2011perturbations}.
			
		\end{proof}

		\bibliographystyle{plain}

	\end{document}